\def\year{2020}\relax
\documentclass[letterpaper]{article} 
\usepackage{aaai20}  
\usepackage{times}  
\usepackage{helvet} 
\usepackage{courier}  
\usepackage[hyphens]{url}  
\usepackage{graphicx} 
\usepackage{graphicx}  
\usepackage[utf8]{inputenc} 
\usepackage{url}            
\usepackage{booktabs}       
\usepackage{amsfonts}       
\usepackage{nicefrac}       
\usepackage{microtype}      
\usepackage{graphics}
\usepackage{graphicx}
\usepackage{subfigure}
\usepackage{bm}
\usepackage{algorithm}
\usepackage{algorithmic}
\newtheorem{theorem}{Theorem}

\newtheorem{proof}{Proof}

\usepackage{float}
\usepackage{graphicx}
\usepackage{amsmath}
\usepackage{subfigure}

\newtheorem{corollary}{Corollary}
\usepackage{booktabs}
\usepackage{amsfonts}
\usepackage{courier}

\title{Deterministic Value-Policy Gradients}
\author{Qingpeng Cai$^{1,}$\thanks{The first two authors contributed equally to this work.}, Ling Pan$^{2,*}$, Pingzhong Tang$^{2}$ \\
\textsuperscript{1} Alibaba Group\\
\textsuperscript{2} IIIS, Tsinghua University\\
qingpeng.cqp@alibaba-inc.com, pl17@mails.tsinghua.edu.cn, kenshin@tsinghua.edu.cn
}
 
\begin{document}

\maketitle

\begin{abstract}

Reinforcement learning algorithms such as the deep deterministic policy gradient algorithm (DDPG) has been widely used in continuous control tasks. However, the model-free DDPG algorithm suffers from high sample complexity. In this paper we consider the deterministic value gradients to improve the sample efficiency of deep reinforcement learning algorithms. Previous works consider deterministic value gradients with the finite horizon, but it is too myopic compared with infinite horizon. We firstly give a theoretical guarantee of the existence of the value gradients in this infinite setting. 
Based on this theoretical guarantee, we propose a class of the deterministic value gradient algorithm (DVG) with infinite horizon, and different rollout steps of the analytical gradients by the learned model trade off between the variance of the value gradients and the model bias. Furthermore, to better combine the model-based deterministic value gradient estimators with the model-free deterministic policy gradient estimator, we propose the deterministic value-policy gradient (DVPG) algorithm. We finally conduct extensive experiments comparing DVPG with state-of-the-art methods on several standard continuous control benchmarks. Results demonstrate that DVPG substantially outperforms other baselines.
\end{abstract}

\section{Introduction}
Silver et al. propose the deterministic policy gradient (DPG) algorithm \cite{silver2014deterministic} that aims to find an optimal deterministic policy that maximizes the expected long-term reward, which lowers the variance when estimating the policy gradient, compared to stochastic policies  \cite{sutton2000policy}.  
Lillicrap et al. further combine deep neural networks with DPG to improve the modeling capacity, and propose the deep deterministic policy gradient (DDPG) algorithm \cite{lillicrap2015continuous}.
It is recognized that DDPG has been successful in robotic control tasks such as locomotion and manipulation.  
Despite the effectiveness of DDPG in these tasks, it suffers from the high sample complexity problem \cite{schulman2015trust}.

Deterministic value gradient methods \cite{werbos1990menu,nguyen1990neural,jordan1992forward,fairbank2008reinforcement} compute the policy gradient through back propagation of the reward along a trajectory predicted by the learned model, which enables better sample efficiency.  
However, to the best of our knowledge, existing works of deterministic value gradient methods merely focus on finite horizon, which are too myopic and can lead to large bias.
Stochastic value gradient (SVG) methods \cite{heess2015learning}  use the re-parameterization technique to optimize the stochastic policies. 
Among the class of SVG algorithms, although SVG($1$) studies infinite-horizon problems, it only uses one-step rollout, which limits its efficiency.
Also, it suffers from the high variance due to the importance sampling ratio and the randomness of the policy. 

In this paper, we study the setting with infinite horizon, where both state transitions and policies are deterministic. 
\cite{heess2015learning} gives recursive Bellman gradient equations of deterministic value gradients, but the gradient lacks of theoretical guarantee as the DPG theorem does not hold in this deterministic transition case.  We prove that the gradient indeed exists for a certain set of discount factors. We then derive a closed form of the value gradients.

However, the estimation of the deterministic value gradients is much more challenging. The difficulty of the computation of the gradient mainly comes from the dependency of the gradient of the value function over the state. Such computation may involve infinite times of the product of the gradient of the transition function and is hard to converge. Thus, applying the Bellman gradient equation recursively may incur high instability.

To overcome these challenges, we use model-based approaches to predict the reward and transition function. Based on the theoretical guarantee of the closed form of the value gradients in the setting, we propose a class of deterministic value gradients DVG($k$) with infinite horizon, where $k$ denotes the number of rollout steps.  For each choice of $k$, we use the rewards predicted by the model and the action-value at $k+1$ step to estimate of the value gradients over the state, in order to reduce the instability of the gradient of the value function over the state. Different number of rollout steps maintains a trade-off between the accumulated model bias and the variance of the gradient over the state. The deterministic policy gradient estimator can be viewed as a special case of this class, i.e., it never use the model to estimate the value gradients, and we refer it to DVG($0$).

As the model-based approaches are more sample efficient than model-free algorithms \cite{li2004iterative,levine2013guided}, and the model-based deterministic value gradients may incur model bias \cite{wahlstrom2015pixels}, we consider an essential question:
{\em How to combines the model-based gradients and the model-free gradients efficiently? }

We propose a temporal difference method to ensemble gradients with different rollout steps. The intuition is to ensemble different gradient estimators with geometric decaying weights.
Based on this estimator, we propose the deterministic value-policy gradient (DVPG) algorithm. The algorithm updates the policy by stochastic gradient ascent with the ensembled value gradients of the policy, and the weight maintains a trade-off between sample efficiency and performance. 

To sum up, the main contribution of the paper is as follows:

\begin{itemize}
\item First of all, we provide a theoretical guarantee for the existence of the deterministic value gradients in settings with infinite horizon.

\item Secondly, we propose a novel algorithm that ensembles the deterministic value gradients and the deterministic policy gradients, called deterministic value-policy gradient (DVPG), which effectively combines the model-free and model-based methods. DVPG reduces sample complexity, enables faster convergence and performance improvement.

\item Finally, we conduct extensive experiments on standard benchmarks comparing with DDPG, DDPG with model-based rollouts, the stochastic value gradient algorithm, SVG($1$) and state-of-the-art stochastic policy gradient methods. Results confirm that DVPG significantly outperforms other algorithms in terms of both sample efficiency and performance.

\end{itemize}

\subsection{Related Work}
Model-based algorithms has been widely studied \cite{moldovan2015optimism,montgomery2016guided,ha2018recurrent,hafner2018learning,chua2018deep,zhang2018solar} in recent years. Model-based methods allows for more efficient computations and faster convergence than model-free methods \cite{wang2003model,li2004iterative,levine2013guided,watter2015embed}. 

There are two classes of model-based methods, one is to use learned model to do imagination rollouts to accelerate the learning.
\cite{gu2016continuous,kurutach2018model} generate synthetic samples by the learned model. PILCO \cite{deisenroth2011pilco} learns the transition model by Gaussian processes and applies policy improvement on analytic policy gradients. The other is to use learned model to get better estimates of action-value functions. The value prediction network (VPN) uses the learned transition model to get a better target estimate \cite{oh2017value}. \cite{feinberg2018model,buckman2018sample} combines different model-based value expansion functions by TD($k$) trick or stochastic distributions to improve the estimator of the action-value function. 
  
Different from previous model-based methods, we present a temporal difference method that ensembles model-based deterministic value gradients and model-free policy gradients. Our technique can be combined with both the imagination rollout technique and the model-based value expansion technique.

\section{Preliminaries}
A Markov decision process (MDP) is a tuple $(S, A, p, r, \gamma, p_0)$, where $\mathcal{S}$ and $\mathcal{A}$ denote the set of states and actions respectively.
$p(s_{t+1} | s_t, a_t)$ represents the conditional density from state $s_t$ to state $s_{t+1}$ under action $a_t$. The density of the initial state distribution is denoted by $p_0(s)$. At each time step $t$, the agent interacts with the environment with a deterministic policy $\mu_{\theta}$.
We use $r(s_t, a_t)$ to represent the immediate reward, contributing to the discounted overall rewards from state $s_0$ following $\mu_{\theta}$, denoted by $J(\mu_{\theta}) = \mathbb{E}[\sum_{k=0}^{\infty}{\gamma}^{k}r(a_k, s_k)|\mu_{\theta},s_0]$.
Here, $\gamma \in [0, 1)$ is the discount factor.
The Q-function of state $s_t$ and action $a_t$ under policy $\mu_{\theta}$ is denoted by $Q^{\mu_{\theta}}(s_t,a_t) =  \mathbb{E}[\sum_{k=t}^{\infty}{\gamma}^{k-t}r(a_k, s_k)|\mu_{\theta},s_t,a_t]$.
The corresponding value function of state $s_t$ under policy $\mu_{\theta}$ is denoted by $V^{\mu_{\theta}}(s_t)=Q^{\mu_{\theta}}(s_t,\mu_{\theta}(s_t))$.
We denote the density at state $s^{'}$ after $t$ time steps from state $s$ following the policy $\mu_{\theta}$ by $p(s,s^{'},t,\mu_{\theta})$ .
We denote the discounted state distribution by $\rho^{\mu_{\theta}}(s^{'})=\int_{\mathcal{S}}^{}\sum_{t=1}^{\infty}{\gamma}^{t-1}p_0(s)p(s,s^{'},t,\mu_{\theta})ds$.
The agent aims to find an optimal policy that maximizes $J(\mu_{\theta})$.

\section{Deterministic Value Gradients}
In this section, we study a setting of infinite horizon with deterministic state transition, which poses challenges for the existence of deterministic value gradients.
We first prove that under proper condition, the deterministic value gradient does exist. 
Based on the theoretical guarantee, we then propose a class of practical algorithms by rolling out different number of steps.
Finally, we discuss the difference and connection between our proposed algorithms and existing works.

Deterministic Policy Gradient (DPG) Theorem \cite{silver2014deterministic}, proves the existence of the deterministic policy gradient for MDP that satisfies the regular condition, which requires the probability density of the next state $p(s^{'}|s,a)$ to be differentiable in $a$. In the proof of the DPG theorem,  the existence of the gradient of the value function is firstly proven, i.e., 
\begin{equation}
\begin{split}
\nabla_{\theta}V^{\mu_{\theta}}(s)=&\int_{\mathcal{S}}\sum_{t=0}^{\infty}{\gamma}^{t}p(s,s',t,\mu_{\theta})\nabla_{\theta}\mu_{\theta}(s')\\
&\nabla_{a'}Q^{\mu_{\theta}}(s',a')|_{a'=\mu_{\theta}(s')}ds',
\end{split}
\end{equation}
then the gradient of the long-terms rewards exists. Without this condition, the arguments in the proof of the DPG theorem do not work \footnote{Readers can refer to http://proceedings.mlr.press/v32/silver14-supp.pdf}, and poses challenges for cases where the differentiability is not satisfied. Note this condition does not hold in any case with deterministic transitions. 
Therefore, one must need a new theoretical guarantee to determine the existence of the gradient of $V^{\mu_{\theta}}(s)$ over $\theta$ in deterministic state transition cases.

\subsection{Deterministic value gradient theorem}
We now analyze the gradient of a deterministic policy.
Denote $T(s,a)$ the next state given current state $s$ and action $a$. Without loss of generality, we assume that the transition function $T$ 
is continuous, differentiable in $s$ and $a$ and is bounded. Note that the regular condition is not equivalent to this assumption. Consider a simple example that a transition $T(s,a)=s+a$, then the gradient of $p(s'|s,a)$ over $a$ is infinite or does not exist. However, the gradient of $T(s,a)$ over $a$ exists. By definition, 

\begin{equation*}
\begin{split}
\nabla_{\theta} V^{\mu_{\theta}}(s) = &\nabla_{\theta} \left( r \left( s,\mu_{\theta}(s) \right)+\gamma  V^{\mu_{\theta}}(s^{'})|_{s^{'}=T(s,\mu_{\theta}(s))} \right)\\
=& \nabla_{\theta} r(s,\mu_{\theta}(s))+\gamma  \nabla_{\theta}V^{\mu_{\theta}}(s^{'})|_{s^{'}=T(s,\mu_{\theta}(s))}\\
&+ \gamma  \nabla_{\theta} T(s,\mu_{\theta}(s)) \nabla_{s'} V^{\mu_{\theta}}(s^{'}).
\end{split}
\end{equation*}

Therefore, the key of the existence (estimation) of the gradient of $V^{\mu_{\theta}}(s)$ over $\theta$ is the existence (estimation) of $\nabla_{s}V^{\mu_{\theta}}(s)$. 
In Theorem \ref{theorem1}, we give a sufficient condition of the existence of $\nabla_{s}V^{\mu_{\theta}}(s)$. 

\begin{theorem}
\label{theorem1}
For any policy $\mu_{\theta}$,  the gradient of the value function over the state, $\nabla_{s}V^{\mu_{\theta}}(s)$, exists with two assumptions:

\begin{itemize}
\item A.1: The set of states that the policy visits starting from any initial state $s$ is finite.
\item A.2: For any initial state $s$, by Assumption A.1, we get that there is a periodic loop of visited states.
Let $(s_0,s_1,...,s_k)$ denote the loop, and $A(s)={\gamma}^{k+1} \prod_{i=0}^{k}\nabla_{s_i}T(s_i,\mu_{\theta}(s_i))$, the power sum of $A(s)$, $\sum_{m=0}^{\infty}{A^{m}(s)}$ converges.
\end{itemize}
\end{theorem}

\begin{proof}
By definition,
\begin{equation}
V^{\mu_{\theta}}(s)=r(s,\mu_{\theta}(s))+\gamma V^{\mu_{\theta}}(s^{'})|_{s^{'}=T(s,\mu_{\theta}(s))}. 
\label{eq: v_def}
\end{equation}
Taking the gradient of Eq. (\ref{eq: v_def}), we obtain
\begin{equation}
\label{v_s}
\begin{split}
\bigtriangledown_{s}V^{\mu_{\theta}}(s)=&\bigtriangledown_{s}r(s,\mu_{\theta}(s)) \\
+&\gamma \bigtriangledown_{s}T(s,\mu_{\theta}(s)) \bigtriangledown_{s^{'}}V^{\mu_{\theta}}(s^{'})|_{s^{'}=T(s,\mu_{\theta}(s))}.
\end{split}
\end{equation}
Unrolling Eq. (\ref{v_s}) with infinite steps, we get

\begin{equation}
\label{sum_of_series}
\begin{split}
\bigtriangledown_{s}V^{\mu_{\theta}}(s)=&\sum_{t=0}^{\infty}{\gamma}^{t}g(s,t,\mu_{\theta})\bigtriangledown_{s_t}r(s_t,\mu_{\theta}(s_t)),
\end{split}
\end{equation}
where $g(s,t,\mu_{\theta})=\prod_{i=0}^{t-1}\bigtriangledown_{s_i}T(s_i,\mu_{\theta}(s_i)),$ $s_0=s$ and $s_i$ is the state after $i$ steps following policy $\mu_{\theta}$. 

With the assumption A.1, we rewrite (\ref{sum_of_series}) by the indicator function  $I(s,s^{'},t,\mu_{\theta})$ that indicates whether $s^{'}$ is obtained after $t$ steps from the initial state $s$ following the policy $\mu_{\theta}$:
\begin{equation}
\begin{split}
\bigtriangledown_{s}V^{\mu_{\theta}}(s)=&\sum_{t=0}^{\infty}\sum_{s'\in B(s,\theta)}^{}{\gamma}^{t}g(s,t,\mu_{\theta})I(s,s^{'},t,\mu_{\theta})\\
&\bigtriangledown_{s^{'}}r(s^{'},\mu_{\theta}(s^{'})),
\end{split}
\end{equation}

Where $B(s,\theta)$ is the set of states the policy visits from $s$.

We now prove that for any $\mu_{\theta},s,s^{'}$, the infinite sum of gradients, $\sum_{t=0}^{\infty}{\gamma}^{t}g(s,t,\mu_{\theta})I(s,s^{'},t,\mu_{\theta})$ converges.

For each state $s'$, there are three cases during the process from the initial state $s$ with infinite steps:
\begin{enumerate}
\item Never visited: $\sum_{t=0}^{\infty}{\gamma}^{t}g(s,t,\mu_{\theta})I(s,s^{'},t,\mu_{\theta})=\mathbf{0}.$
\item Visited once: Let $t_{s'}$ denote the number of steps that it takes to reach the state $s'$, then 
$\sum_{t=0}^{\infty}{\gamma}^{t}g(s,t,\mu_{\theta})I(s,s^{'},t,\mu_{\theta})={\gamma}^{t_{s^{'}}}g(s,t_{s^{'}},\mu_{\theta}).$
\item Visited infinite times: Let $t_1$ denote the number of steps it takes to reach $s'$ for the first time. 
The state $s'$ will be revisited every $k$ steps after the previous visit.  By definition,

\begin{equation}
\begin{split}
\label{sum}
\sum_{t=0}^{\infty}{\gamma}^{t}g(s,t,\mu_{\theta})I(s,s^{'},t,\mu_{\theta}) 
=\sum_{a=0}^{\infty}{\gamma}^{t_1}g(s,t_1,\mu_{\theta}){A^{a}(s)}.
\end{split}
\end{equation}

By the assumption A.2 we get (\ref{sum}) converges. 
\end{enumerate}

By exchanging the order of the limit and the summation,
\begin{equation}
\begin{split}
\bigtriangledown_{s}V^{\mu_{\theta}}(s)=&\sum_{s'\in B(s,\theta)}^{}\sum_{t=0}^{\infty}{\gamma}^{t}g(s,t,\mu_{\theta})I(s,s^{'},t,\mu_{\theta})\\
&\bigtriangledown_{s^{'}}r(s^{'},\mu_{\theta}(s^{'})).
\end{split}
\end{equation}

\end{proof}

Assumption A.1 guarantees the existence of the stationary distribution of states theoretically. 
Actually, it holds on most continuous tasks, e.g., InvertedPendulum-v2 in MuJoCo. 
We directly test a deterministic policy with  a 2-layer fully connected network on this environment with 10,000 episodes\footnote{We test different weights, the observation of finite visited states set is very common among different weights.}, and we count the number that each state is visited. 
After projecting the data into 2D space by t-SNE \cite{maaten2008visualizing}, we obtain the state visitation density countour as shown in Figure \ref{fig: therodynamics}. 
We have two interesting findings: (1) The set of states visited by the policy is finite. (2) Many states are visited for multiple times, which justifies Assumption A.1.

By the analysis of Assumption A.2, we get that for any policy and state, there exists a set of discount factors such that the the gradient of the value function over the state exists, as illustrated in Corollary \ref{cor1}. Please refer to Appendix A for the proof.

\begin{corollary}
For any policy $\mu_{\theta}$ and any initial state $s$, let $(s_0,s_1,...,s_k)$ denote the loop of states following the policy and the state, $C(s,\mu_{\theta},k)=\prod_{i=0}^{k}\nabla_{s_i}T(s_i,\mu_{\theta}(s_i))$, the gradient of the value function over the state, $\nabla_{s}V^{\mu_{\theta}}(s)$ exists if 
${\gamma}^{k+1}\max\left\{||C(s,\mu_{\theta},k)||_{\infty}, ||C(s,\mu_{\theta},k)||_{1}\right\}<1.$
\label{cor1}
\end{corollary}

\begin{figure}
    \centering
            \includegraphics[scale=0.26]{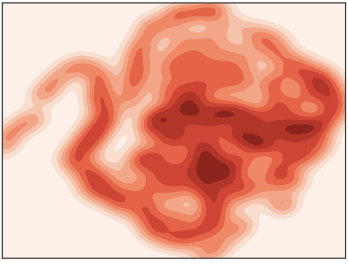}
    \caption{State visitation density countour on InvertedPendulum-v2.}
    \label{fig: therodynamics}
\end{figure}

In Theorem \ref{theorem_deter}, we show that the deterministic value gradients exist and obtain the closed form based on the analysis in Theorem \ref{theorem1}. Please refer to Appendix B for the proof. 

\begin{theorem} {\bf{(Deterministic Value Gradient Theorem)}}
\label{theorem_deter}
For any policy $\mu_{\theta}$ and MDP with deterministic state transitions, if assumptions A.1 and A.2 hold, the value gradients exist, and
\begin{equation*}
\label{close_j}
\begin{split}
\nabla_{\theta}V^{\mu_{\theta}}(s)=&\sum_{s'\in B(s,\theta)}^{}\rho^{\mu_{\theta}}(s,s')\nabla_{\theta}\mu_{\theta}(s')(\nabla_{a'}r(s',a')
+\\
&\gamma  \nabla_{a'} T(s',a') \nabla_{s^{''}}V^{\mu_{\theta}}(s^{''})|_{s^{''}=T(s',a')}),
\end{split}
\end{equation*}
where $ a'$ is the action the policy takes at state $s'$, $\rho^{\mu_{\theta}}(s,s')$ is the discounted state distribution starting from the state $s$ and the policy, and is defined as
$
\rho^{\mu_{\theta}}(s,s')=\sum_{t=1}^{\infty}{\gamma}^{t-1}I(s,s',t,\mu_{\theta}).$
\end{theorem}

\subsection{Deterministic value gradient algorithm}
The value gradient methods estimate the gradient of value function recursively \cite{fairbank2012value}:

\begin{equation}
\begin{split}
\nabla_{\theta} V^{\mu_{\theta}}(s) =& \nabla_{\theta} r(s, \mu_{\theta}(s)) + \gamma \nabla_{\theta} T(s,\mu_{\theta}(s)) \nabla_{s'}V^{\mu_{\theta}}(s')\\
 &+ \gamma \nabla_{\theta} V^{\mu_{\theta}}(s')
\end{split}
\label{policy graident}
\end{equation}

\begin{equation}
\label{v_snew}
\begin{split}
\bigtriangledown_{s}V^{\mu_{\theta}}(s)=&\bigtriangledown_{s}r(s,\mu_{\theta}(s))+\gamma \bigtriangledown_{s}T(s,\mu_{\theta}(s))\\
&\bigtriangledown_{s^{'}}V^{\mu_{\theta}}(s^{'})|_{s^{'}=T(s,\mu_{\theta}(s))}.
\end{split}
\end{equation}

In fact, there are two kinds of approaches for estimating the gradient of the value function over the state, i.e., infinite and finite.
On the one hand, directly estimating the gradient of the value function over the state recursively by Eq. (\ref{v_snew}) for infinite times is slow to converge.
On the other hand, estimating the gradient by finite horizon like traditional value gradient methods \cite{werbos1990menu,nguyen1990neural,heess2015learning} may cause large bias of the gradient.

We set out to estimate the action-value function denoted by $Q^{w}(s,a)$ with parameter $w$,  and replace $\nabla_{s}V^{\mu_{\theta}}(s)$ by $ \nabla_{s}Q^{w}(s,\mu_{\theta}(s))$ in Eq. (\ref{policy graident}). 
In this way, we can directly obtain a 1-step estimator of the value gradients, 
\begin{equation}
\label{DVG(1)}
\begin{split}
G_1(\mu_{\theta},s)=& \nabla_{\theta} r(s,\mu_{\theta}(s)) + \gamma  \nabla_{\theta} T(s,\mu_{\theta}(s))\\
 &\nabla_{s_{1}} Q^{w}(s_1,\mu_{\theta}(s_1)) + \gamma  G_1(\mu_{\theta},s_{1}),
\end{split}
\end{equation}

where $s_1$ is the next state of $s$, which can be generalized to $k (k\geq 2)$ rollout steps.  Let $s_i$ denote the state visited by the policy at the $i$-th step starting form the initial state $s_0$, $g(s,t,\mu_{\theta},T)=\prod_{i=1}^{t-1}\nabla_{s_i}T(s_i,\mu_{\theta}(s_i)).$  We choose to rollout $k-1$ steps to get rewards, then replace $\nabla_{s_k}V^{\mu_{\theta}}(s_k)$ by $ \nabla_{s_k}Q^{w}(s_k,\mu_{\theta}(s_k))$ in Eq. (\ref{v_snew}), and we get

\begin{equation*}
\label{DVG(value_state)}
\begin{split}
L_k(\mu_{\theta},s,r,T) =&\sum_{t=1}^{k-1}{\gamma}^{t-1}g(s,t,\mu_{\theta},T)\nabla_{s_t}r(s_t,\mu_{\theta}(s_t))\\
&+ {\gamma}^{k-1} g(s,k,\mu_{\theta},T)\nabla_{s_{k}} Q^{w}(s_{k},\mu_{\theta}(s_{k})).
\end{split}
\end{equation*}

Replacing $\nabla_{s'}V^{\mu_{\theta}}(s')$ with $L_k(\mu_{\theta},s,r,T)$ in Eq. (\ref{policy graident}), we get a $k$-step estimator of the value gradients:

\begin{equation}
\label{DVG(k)}
\begin{split}
G_k(\mu_{\theta},s)=& \nabla_{\theta}r(s,\mu_{\theta}(s))+  \gamma  \nabla_{\theta} T(s,\mu_{\theta}(s)\\
& L_k(\mu_{\theta},s,r,T) + \gamma  G_k(\mu_{\theta},s_1).
\end{split}
\end{equation}

It is easy to see that $G_k(\mu_{\theta},s)$ and $G_1(\mu_{\theta},s)$ are the same if we have the true reward and transition functions, which is generally not the case as we need to learn the model in practical environments.
Let  $D_k(\mu_{\theta},s, T', r')$ denote the value gradient at the sampled state $s$ with $k$ rollout steps, on learned transition function $T'$ and reward function $r'$, which is defined as:
\begin{equation}
\label{DVG(gradient single)}
\begin{split}
 D_k(\mu_{\theta},s, T', r') =& \nabla_{\theta}r'(s,\mu_{\theta}(s))+  \gamma  \nabla_{\theta} T'(s,\mu_{\theta}(s))\\
 &L_k(\mu_{\theta},s,r',T').
\end{split}
\end{equation}
Based on Eq.(\ref{DVG(gradient single)}), we propose the deterministic value gradients with infinite horizon, where the algorithm is shown in Algorithm \ref{alg:DVG}:
given $n$ samples $(s_j,a_j,r_j,s_{j+1})$, for each choice of $k$, 
we use 
$\frac{1}{n}\sum_{j}^{} D_k(\mu_{\theta},s_j, T', r')$
to update the current policy. We use sample-based methods to estimate the deterministic value gradients. For each state in the trajectory, we take the analytic gradients by the learned model. 
As the model is not given, we choose to predict the reward function and the transition function.
We choose to use experience replay to compare with the DDPG algorithm fairly.
Different choices of the number of rollout steps trade-off between the variance and the bias.
Larger steps means lower variance of the value gradients, and larger bias due to the accumulated model error.

\begin{algorithm}
\small
  \caption{The DVG($k$) algorithm}
  \begin{algorithmic}[1]
   \STATE Initialize the reward network $r'$, transition network $T'$, critic network $Q$, actor network $\mu_{\theta}$, target networks $Q'$, $\mu_{\theta}'$ and experience replay buffer $\mathcal{B}$\\
  \FOR{episode$=0, ..., N-1$} 

    \FOR{$t=1, ..., T$} 
	\STATE	Select action according to the current policy and exploration noise\\
	\STATE	Execute action $a_t$, observe reward $r_t$ and new state $s_{t+1}$, and store transition $(s_t,a_t,r_t,s_{t+1})$ in $\mathcal{B}$\\
	\STATE	Sample a random minibatch of $n$ transitions from $\mathcal{B}$\\
	\STATE	Update the critic $Q$ by minimizing the TD error: $\frac{1}{n}\sum_{j}^{n}{(r_j+\gamma Q'(s_{j+1},\mu_{\theta}(s_{j+1})) - Q(s_j,a_j))}^{2}$\\
	\STATE	Update the reward network $r'$ and the transition network $T'$ on the batch by minimizing the square loss\\
	\STATE Estimate the value gradients by $\frac{1}{n}\sum_{j}^{} D_k(\mu_{\theta},s_j, T',r')$  and perform gradient update on the policy\\
    \STATE Update the target networks by ${\theta}_{}^{Q^{'}}=\tau {\theta}_{}^{Q}+(1-\tau) {\theta}_{}^{Q^{'}}; {\theta}^{\mu^{'}}=\tau {\theta}^{\mu}+(1-\tau) {\theta}^{\mu^{'}}$\\
    \ENDFOR
  \ENDFOR
\end{algorithmic}
  \label{alg:DVG}
\end{algorithm}

\subsection{The difference between infinite and finite horizon}
In this section, we discuss the advantage of our proposed DVG algorithm over finite horizon and validate the effect on a continuous control task. 
The estimator of deterministic value gradients with finite horizon, DVG$_{\text F}$, is defined as \cite{fairbank2012value}:

\begin{equation*}
\label{DVG_Finite}
\begin{split}
F_k(\mu_{\theta},s)=& \nabla_{\theta}r(s,\mu_{\theta}(s))+  \gamma  \nabla_{\theta} T(s,\mu_{\theta}(s)) \sum_{t=1}^{k-1}{\gamma}^{t-1}\\
&g(s,t,\mu_{\theta},T)\nabla_{s_t}r(s_t,\mu_{\theta}(s_t))+ \gamma  F_k(\mu_{\theta},s_1).
\end{split}
\end{equation*}

Note that $F_k(\mu_{\theta},s)$ does not take rewards after the $k$-th step into consideration. 
Therefore, given $n$ samples $\{(s_j,a_j,r_j,s_{j+1})\}$, DVG$_{\text F}$ uses the sample mean of $D'_k(\mu_{\theta},s, T', r')$ to update the policy, where $D'_k(\mu_{\theta},s, T', r')$ is defined as:
\begin{equation*}
\label{DVG inf (gradient single)}
\begin{split}
 D'_k(\mu_{\theta},s, T', r') = &\nabla_{\theta}r'(s,\mu_{\theta}(s))+  \gamma  \nabla_{\theta} T'(s,\mu_{\theta}(s))\\
 &\sum_{t=1}^{k-1}{\gamma}^{t-1}g(s,t,\mu_{\theta},T')\nabla_{s_t^{'}}r(s_t^{'},\mu_{\theta}(s_t^{'})).
\end{split}
\end{equation*}

We then test the two approaches on the environment HumanoidStandup-v2, where we choose the parameter $k$ to be $2$\footnote{For the choice of $k$, we test DVG$_{\text F}$ with steps ranging from 1 to 5, and we choose the parameter with the best performance for fair comparison.}.  As shown in Figure \ref{fig: finite}, DVG significantly outperforms DVG$_{\text F}$, which validates our claim that only considering finite horizon fails to achieve the same performance as that of infinite horizon.

\begin{figure}[H]
    \centering
            \includegraphics[scale=0.32]{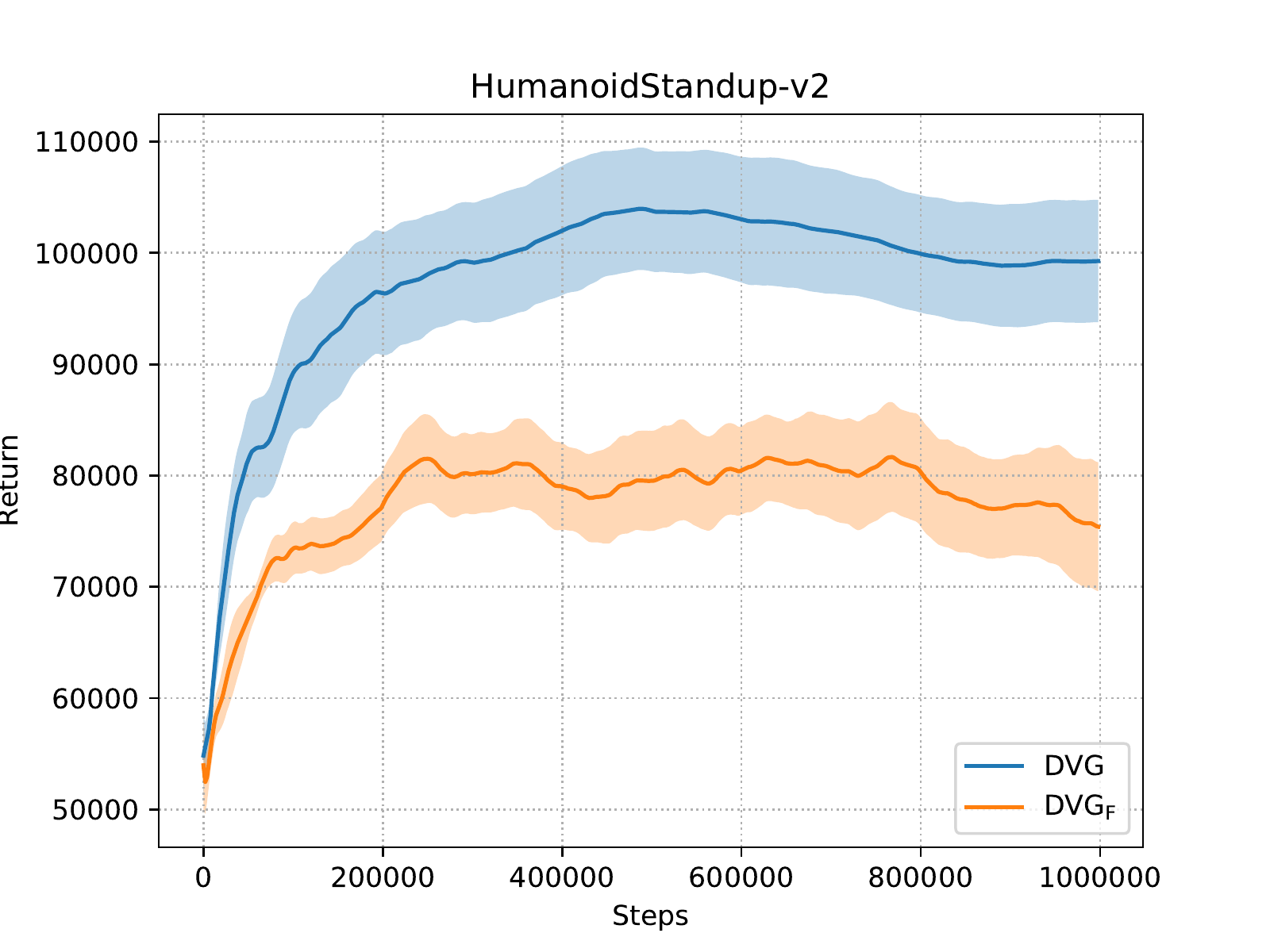}
    \caption{Comparisons of DVG and DVG$_{\text F}$.}
    \label{fig: finite}
\end{figure}

\begin{figure}[H]
    \centering
            \includegraphics[scale=0.32]{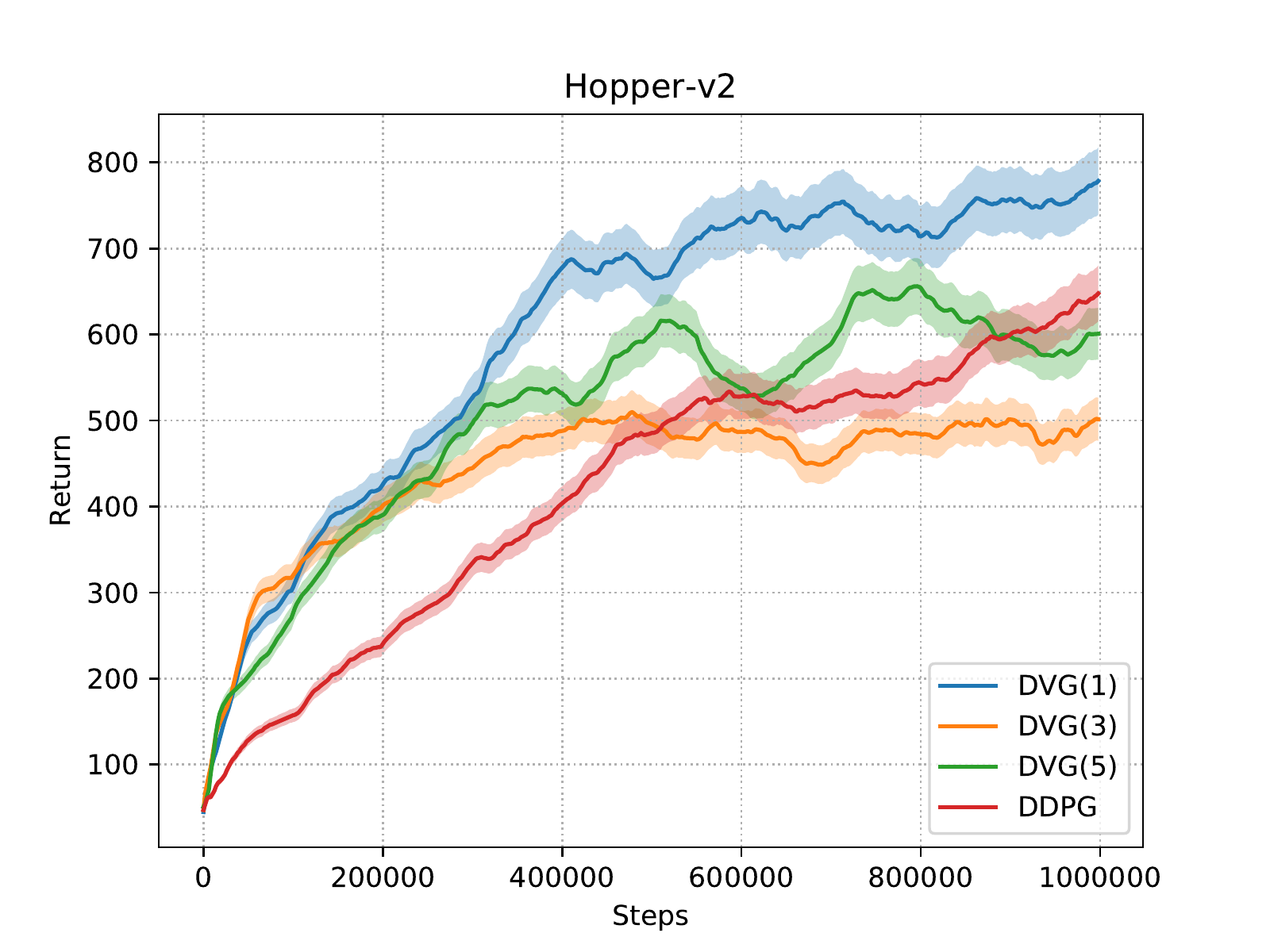}
    \caption{Comparisons of DVG with DDPG.}
    \label{fig: just}
\end{figure}

\subsection{Connection and comparison of DVG and DDPG}
By the proof of the DPG theorem in \cite{silver2014deterministic}, Eq. (\ref{policy graident}) can be re-written as
\begin{equation}
\label{DPG}
\begin{split}
\nabla_{\theta}V^{\mu_{\theta}}(s)= \nabla_{\theta} \mu_{\theta}(s)  \nabla_{a} Q^{\mu_{\theta}}(s,a) + \gamma  \nabla_{\theta}V^{\mu_{\theta}}(s_1).
\end{split}
\end{equation}

The DDPG algorithm uses the gradient of the estimator of the Q function over the action, $\nabla_{a}Q^{w}(s,a)$  to estimate $\nabla_{a} Q^{\mu_{\theta}}(s,a)$, i.e.,
$G_0(\mu_{\theta},s)= \nabla_{\theta} \mu_{\theta}(s)  \nabla_{a} Q^{w}(s,a) + \gamma G_0(\mu_{\theta},s_1).$

The DDPG algorithm is a model-free algorithm which does not predict the reward and the transition, and can be viewed as the DVG($0$) algorithm.
We compare the DVG algorithm with different rollout steps $k$ and DDPG on a continuous control task in MuJoCo, Hopper-v2. 
From Figure \ref{fig: just}, we get that DVG with any choice of the number of rollout steps is more sample efficient than DDPG, which validates the power of model-based techniques.
DVG($1$) outperforms DDPG and DVG with other number of rollout steps in terms of performance as it trades off well between the bias and the variance of the value gradients. 
Note that with a larger number of step, DVG(5) is not stable due to the propagated model error.

\section{The DVPG Algorithm}
As discussed before, the model-based DVG algorithm are more sample efficient than the model-free DDPG algorithm. However, it suffers from the model bias which results in performance loss. In this section, we consider to ensemble these different gradient estimators for better performance. 

Motivated by the idea of TD($\lambda$) algorithm \cite{sutton2018reinforcement}, which ensembles the TD($k$) error with a geometric decaying weight $\lambda$, we propose a temporal-difference method to ensemble DVG with varying rollout steps and the model-free deterministic policy gradients. We defined the temporal difference deterministic value gradients as 
$G_{\lambda,t}(\mu_{\theta},s)=(1-\lambda) \sum_{k=0}^{t} {\lambda}^{k} G_{k}(\mu_{\theta},s)$,
where $t$ denotes the maximal number of rollout steps by the learned model. For the gradient update rule, we also apply sample based methods: given $n$ samples $\{(s_j,a_j,r_j,s_{j+1})\}$, we use 

\begin{equation}
\label{TDDVG}
\begin{split}
&\frac{1}{n}\sum_{j}^{}( (1-\lambda) \nabla_{\theta} \mu_{\theta}(s_j)  \nabla_{a} Q^{w}(s_j,a) +
 (1-\lambda)\\
&\sum_{k=1}^{t} {\lambda}^{k} D_k(\mu_{\theta},s_j, T', r'))
\end{split}
\end{equation}
to update the policy. Based on this ensembled deterministic value-policy gradients, we propose the deterministic value-policy gradient algorithm, shown in Algorithm \ref{alg:DVPG} \footnote{The only difference between the DVG(k) algorithm and the DVPG algorithm is the update rule of the policy.}.

\begin{algorithm}
\small
  \caption{The DVPG algorithm}
  \begin{algorithmic}[1]
  \STATE Initialize the weight $\lambda$ and the number of maximal steps $t$\\
 \STATE Initialize the reward network $r'$, transition network $T'$, critic network $Q$, actor network $\mu_{\theta}$, target networks $Q'$, $\mu_{\theta}'$ and experience replay buffer $\mathcal{B}$\\
  \FOR{episode$=0, ..., N-1$} 
    \FOR{$t=1, ..., T$} 
	\STATE	Select action according to the current policy and exploration noise\\
	\STATE	Execute action $a_t$, observe reward $r_t$ and new state $s_{t+1}$, and store transition $(s_t,a_t,r_t,s_{t+1})$ in $\mathcal{B}$\\
	\STATE	Sample a random minibatch of $n$ transitions from $\mathcal{B}$\\
	\STATE	Update the critic $Q$ by minimizing the TD error: $\frac{1}{n}\sum_{j}^{n}{(r_j+\gamma Q'(s_{j+1},\mu_{\theta}(s_{j+1})) - Q(s_j,a_j))}^{2}$\\
	\STATE	Update the reward network $r'$ and the transition network $T'$ on the batch by minimizing the square loss\\
	\STATE Estimate the value gradients by Eq. (\ref{TDDVG}), and perform gradient update on the policy \\
    \STATE Update the target networks by ${\theta}_{}^{Q^{'}}=\tau {\theta}_{}^{Q}+(1-\tau) {\theta}_{}^{Q^{'}}; {\theta}^{\mu^{'}}=\tau {\theta}^{\mu}+(1-\tau) {\theta}^{\mu^{'}}$\\
    \ENDFOR
  \ENDFOR
\end{algorithmic}
  \label{alg:DVPG}
\end{algorithm}

\section{Experimental Results}

We design a series of experiments to evaluate DVG and DVPG. We investigate the following aspects:
(1) What is the effect of the discount factor on DVG?
(2) How sensitive is DVPG to the hyper-parameters?
(3) How does DVPG compare with state-of-the-art methods?

We evaluate DVPG in a number of continuous control benchmark tasks in OpenAI Gym based on the MuJoCo \cite{todorov2012mujoco} simulator.
The implementation details are referred to Appendix C. 
We compare DVPG with DDPG, DVG, DDPG with imagination rollouts (DDPG(model)), and SVG with 1 step rollout and experience replay (SVG($1$)) in the text. 
We also compare DVPG with methods using stochastic policies, e.g. ACKTR, TRPO, in Appendix D. 
We plot the averaged rewards of episodes over 5 different random seeds with the number of real samples, and the shade region represents the 75\% confidence interval. 
We choose the same hyperparameters of the actor and critic network for all algorithms. The prediction models of DVPG, DVG and DDPG(model) are the same. 

\begin{figure*}
    \centering
            \includegraphics[scale=0.245]{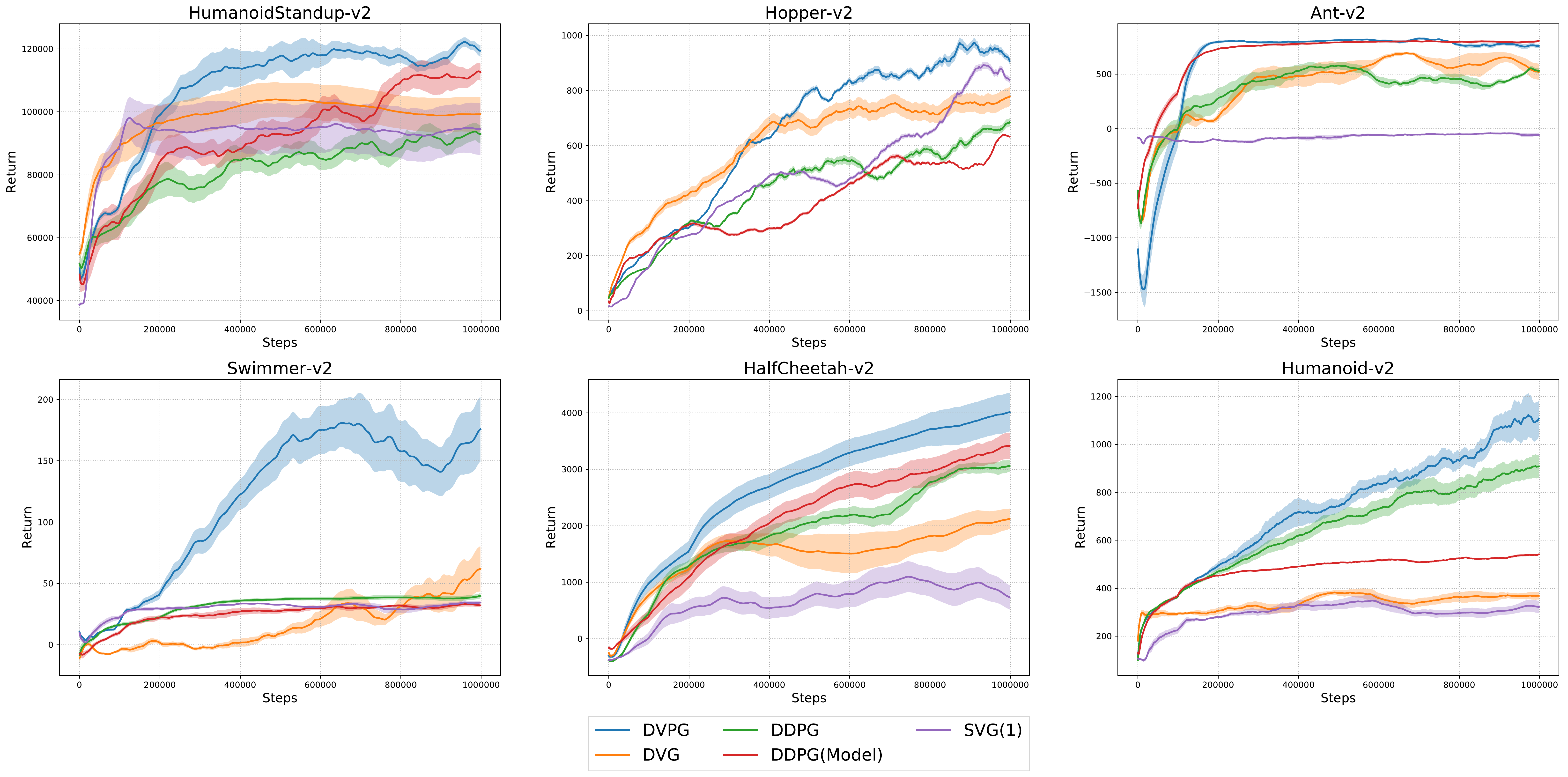}
    \caption{Performance comparisons on environments from the MuJoCo simulator.}
    \label{fig: mujoco}
\end{figure*}

\subsection{The effect of discount factors on DVG}

From Eq. (\ref{v_snew}), we get that $\nabla_{s}V^{\mu_{\theta}}(s)$ is equivalent to the infinite sum of the gradient vectors.
To study the effect of the discount factor on DVG, we train the algorithm with 2 rollout steps with different values of the discount factor on the environment InvertedPendulum-v2. 
As shown in Figure \ref{fig: discount}, $0.95$ performs the best in terms of rewards and stability while $0.85$ and $0.99$ performs comparably, while the performance of $0.8$ and $0.6$ are inferior to other values.
This is because the convergence of the computation of the gradient of the value function over the state may be slow if the discount factor is close to 1 while a smaller value of $\gamma$ may enable better convergence of  $\nabla_{s} V^{\mu_{\theta}}(s)$. 
However, the sum of rewards discounted by a too small $\gamma$ will be too myopic, and fails to perform good.
Here, $0.95$ trades-off well between the stability and the performance, which is as expected that there exists an optimal intermediate value for the discount factor.  

\begin{figure}[H]
    \centering
            \includegraphics[scale=0.32]{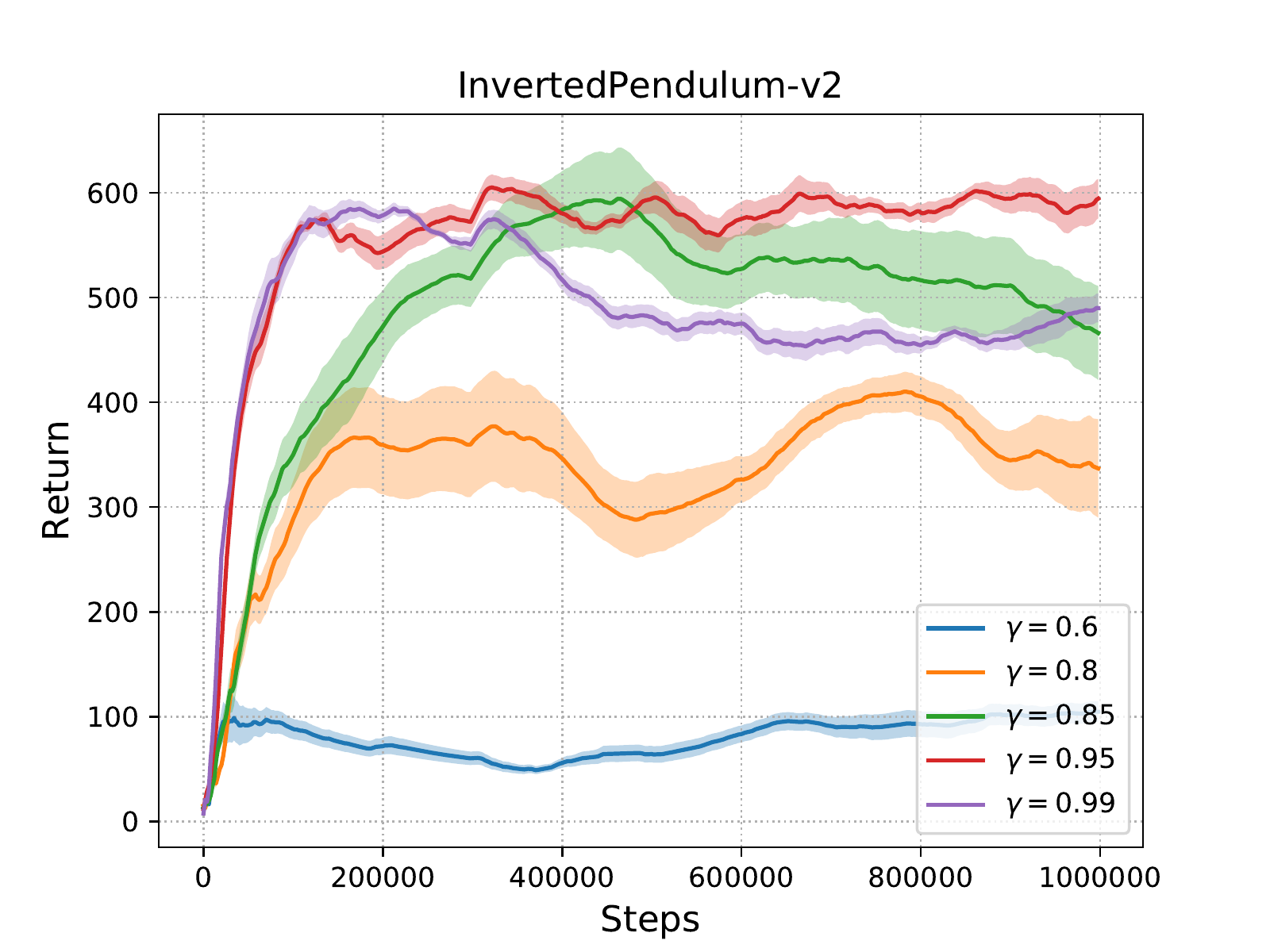}
    \caption{The effect of discount factors.}
    \label{fig: discount}
\end{figure}

\subsection{ Ablation study of DVPG}
We evaluate the effect of the weight of bootstrapping on DVPG with different values from $0.1$ to $0.9$, where the number of rollout steps is set to be 4.
From Figure \ref{fig: bootstrap}, we get that the performance of the DVPG decreases with the increase of the value $\lambda$, where $0.1$ performs the best in terms of the sample efficiency and the performance. 
Thus, we choose the value of the weight to be $0.1$ in all experiments.

\begin{figure}[H]
    \centering
            \includegraphics[scale=0.32]{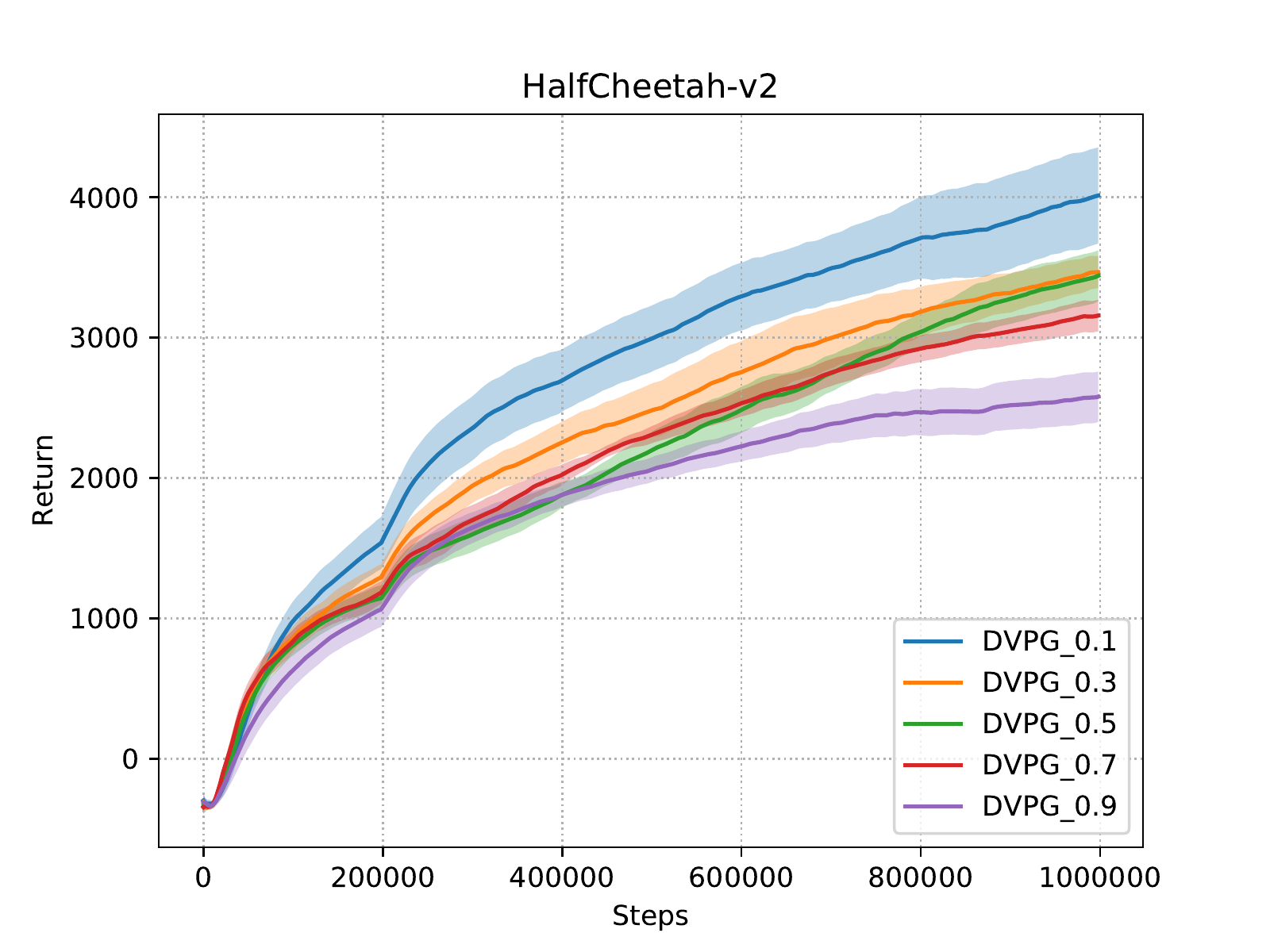}
    \caption{The weight of bootstrapping.}
    \label{fig: bootstrap}
\end{figure}

We evaluate the effect of the number of rollout steps ranging from $1$ to $5$. Results in Figure \ref{fig: rollout steps} show that DVPG with different number of rollout steps all succeed to learn a good policy, with 1 rollout step performing the best. Indeed, the number of rollout steps trade off between the model-error and the variance. There is an optimal value of the number of rollout steps for each environment, which is the only one parameter we tune. To summarize, for the number of look steps, 1 rollout step works the best on Humanoid-v2, Swimmer-v2 and HalfCheetah-v2, while 2 rollout steps performs the best on HumanoidStandup-v2, Hopper-v2 and Ant-v2. 
For fair comparisons, we choose the same number of rollout steps for both the DVG and the DVPG algorithm.  

\begin{figure}[H]
    \centering
            \includegraphics[scale=0.32]{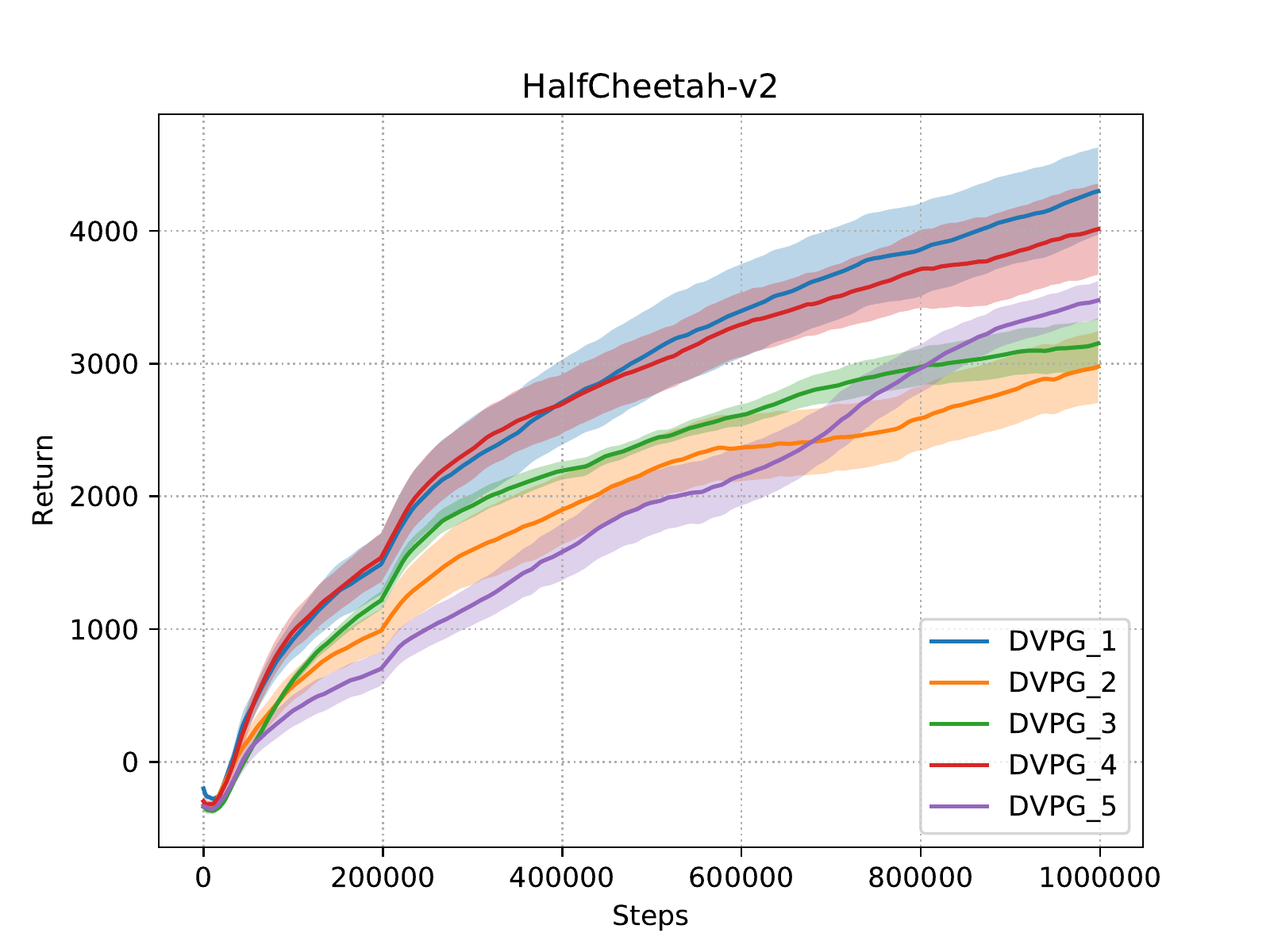}
    \caption{The number of rollout steps.}
    \label{fig: rollout steps}
\end{figure}

\subsection{Performance comparisons}

In this section we compare DVPG with the model-free baseline DDPG, and model-based baselines including DVG, DDPG(model) and SVG($1$) on several continuous control tasks on MuJoCo.  As shown in Figure \ref{fig: mujoco}, there are two classes of comparisons. 

Firstly, we compare DVPG with DDPG and DVG to validate the effect of the temporal difference technique to ensemble model-based and model-free deterministic value gradients. 
The DVG algorithm is the most sample efficient than other algorithms in environments HumanoidStandup-v2, and Hopper-v2. 
For sample efficiency, DVPG outperforms DDPG as it trades off between the model-based deterministic value gradients and the model-free deterministic policy gradients. 
In the end of the training, DVPG outperforms other two algorithms significantly, which demonstrates the power of the temporal difference technique. 
In other four environments, DVPG outperforms other algorithms in terms of both sample efficiency and performance. 
The performance of DVG and DDPG on Swimmer-v2 and Ant-v2 are comparable, while DVG performs bad in Halfcheetah-v2 and Humanoid-v2 due to the model-error. 

Secondly, we compare DVPG with SVG($1$) and DDPG with imagination rollouts.  Results show that the DVPG algorithm significantly outperforms these two model-based algorithms in terms of sample efficiency and performance, especially in environments where other model-based algorithms do not get better performance than the model-free DDPG algorithm. For the performance of the SVG($1$) algorithm, it fails to learn good policies in Ant-v2, which is also reported in \cite{kurutach2018model}.

\section{Conclusion}

Due to high sample complexity of the model-free DDPG algorithm and high bias of the deterministic value gradients with finite horizon, we study the deterministic value gradients with infinite horizon.
We prove the existence of the deterministic value gradients for a certain set of discount factors in this infinite setting. 
Based on this theoretical guarantee, we propose the DVG algorithm with different rollout steps by the model. 
We then propose a temporal difference method to ensemble deterministic value gradients and deterministic policy gradients, to trade off between the bias due to the model error and the variance of the model-free policy gradients, called the DVPG algorithm. We compare DVPG on several continuous control benchmarks. Results show that DVPG substantially outperforms other baselines in terms of convergence and performance.
For future work,  it is promising to apply the temporal difference technique presented in this paper to other model-free algorithms and combine with other model-based techniques.

\section{Acknowledgments}
The work by Ling Pan was supported in part by the National Natural Science Foundation of China Grants 61672316.
Pingzhong Tang was supported in part by the National Natural Science Foundation of China Grant 61561146398, and a China Youth 1000-talent program.

\appendix

\section{A. Proof of Corollary 1}
\textbf{Corollary 1}
\emph{
For any policy $\mu_{\theta}$ and any initial state $s$, let $(s_0,s_1,...,s_k)$ denote the loop of states following the policy and the state, let $C(s,\mu_{\theta},k)=\prod_{i=0}^{k}\nabla_{s_i}T(s_i,\mu_{\theta}(s_i))$, the gradient of the value function over the state, $\nabla_{s}V^{\mu_{\theta}}(s)$ exists if 
\begin{equation}
\begin{split}
{\gamma}^{k+1}\max\left\{||C(s,\mu_{\theta},k)||_{\infty}, ||C(s,\mu_{\theta},k)||_{1}\right\}<1.
\end{split}
\end{equation}
}

\begin{proof}

By the definition of $A(s)$, we get 
\begin{equation}
\begin{split}
\max\left\{||A(s)||_{\infty}, ||A(s)||_{1}\right\}<1.
\end{split}
\end{equation}
Then by \cite{farnell1944limits}, the absolute value of any eigenvalue of $A(s)$ is strictly less than $1$. 
By representing $A(s)$ with Jordan normal form, i.e., $A(s)=MJM^{-1}$,

\begin{equation}
\sum_{m=0}^{\infty}{A^{m}(s)}=M\sum_{m=0}^{\infty}J^{m} M^{-1}.
\end{equation}

As the absolute value of any eigenvalue of $A(s)$ is strictly less than $1$, $\sum_{m=0}^{\infty}J^{m}$ converges, then $\sum_{m=0}^{\infty}{A^{m}(s)}$ converges.
By Lemma 1, $\nabla_{s}V^{\mu_{\theta}}(s)$ converges.

\end{proof}

\section{B. Proof of Theorem 2}
\textbf{Theorem 2}
\emph{
For any policy $\mu_{\theta}$ and MDP with deterministic state transitions, if assumptions A.1 and A.2 hold, the value gradients exist, and
\begin{equation}
\label{close_v}
\begin{split}
\nabla_{\theta}V^{\mu_{\theta}}(s)=&\sum_{s'\in B(s,\theta)}^{}\rho^{\mu_{\theta}}(s,s')\nabla_{\theta}\mu_{\theta}(s')(\nabla_{a'}r(s',a')
+\gamma\\
&  \nabla_{a'} T(s',a') \nabla_{s^{''}}V^{\mu_{\theta}}(s^{''})|_{s^{''}=T(s',a'), {a'=\mu_{\theta}(s')}}),
\end{split}
\end{equation}
where $\rho^{\mu_{\theta}}(s,s')$ is the discounted state distribution starting from the state $s$ and the policy, and is defined as
$
\rho^{\mu_{\theta}}(s,s')=\sum_{t=1}^{\infty}{\gamma}^{t-1}I(s,s^{'},t,\mu_{\theta}).$
}

\begin{proof}
By  definition,
\begin{equation}
\label{V}
\begin{split}
&\nabla_{\theta}V^{\mu_{\theta}}(s)\\
=&\nabla_{\theta}Q^{\mu_{\theta}}(s,\mu_{\theta}(s))\\
=&\nabla_{\theta}(r(s,\mu_{\theta}(s))+\gamma V^{\mu_{\theta}}(s^{'})|_{s^{'}=T(s,\mu_{\theta}(s))})\\
=&\nabla_{\theta}\mu_{\theta}(s)\nabla_{a}r(s,a)|_{a=\mu_{\theta}(s)}+\gamma \nabla_{\theta}V^{\mu_{\theta}}(s^{'})|_{s^{'}=T(s,\mu_{\theta}(s))}\\
+&\gamma \nabla_{\theta}\mu_{\theta}(s) \nabla_{a}T(s,a)|_{a=\mu_{\theta}(s)} \nabla_{s^{'}}V^{\mu_{\theta}}(s^{'})|_{s^{'}=T(s,a)}.
\end{split}
\end{equation}

With the indicator function $I(s,s^{'},t,\mu_{\theta})$, we rewrite the equation (\ref{V}):
\begin{equation}
\label{gra_v}
\begin{split}
&\nabla_{\theta}V^{\mu_{\theta}}(s)\\
=&\nabla_{\theta}\mu_{\theta}(s)(\nabla_{a}r(s,a)|_{a=\mu_{\theta}(s)}\\
&+\gamma \nabla_{a} T(s,a)|_{a=\mu_{\theta}(s)} \nabla_{s^{'}}V^{\mu_{\theta}}(s^{'})|_{s^{'}=T(s,a)})\\
&+ \sum_{s'}^{}\gamma I(s,s^{'},1,\mu_{\theta})\nabla_{\theta}V^{\mu_{\theta}}(s^{'})\\
=&\nabla_{\theta}\mu_{\theta}(s)(\nabla_{a}r(s,a)|_{a=\mu_{\theta}(s)}\\
&+\gamma \nabla_{a} T(s,a)|_{a=\mu_{\theta}(s)} \nabla_{s^{'}}V^{\mu_{\theta}}(s^{'})|_{s^{'}=T(s,a)})\\
&+  \sum_{s'}^{}\gamma I(s,s^{'},1,\mu_{\theta})\nabla_{\theta}\mu_{\theta}(s^{'})(\nabla_{a^{'}}r(s^{'},a^{'})|_{a^{'}=\mu_{\theta}(s^{'})}\\
&+\gamma \nabla_{a^{'}} T(s^{'},a^{'})|_{a^{'}=\mu_{\theta}(s^{'})} \nabla_{s^{''}}V^{\mu_{\theta}}(s^{''})|_{s^{''}=T(s^{'},a^{'})})ds^{'}\\
&+  \sum_{s'}^{}\gamma I(s,s^{'},1,\mu_{\theta}) \sum_{s''}^{}\gamma I(s^{'},s^{''},1,\mu_{\theta})\nabla_{\theta}V^{\mu_{\theta}}(s^{''})\\
=&\nabla_{\theta}\mu_{\theta}(s)(\nabla_{a}r(s,a)|_{a=\mu_{\theta}(s)}\\
&+\gamma \nabla_{a} T(s,a)|_{a=\mu_{\theta}(s)} \nabla_{s^{'}}V^{\mu_{\theta}}(s^{'})|_{s^{'}=T(s,a)})\\
&+  \sum_{s'}^{}\gamma I(s,s^{'},1,\mu_{\theta})\nabla_{\theta}\mu_{\theta}(s^{'})(\nabla_{a^{'}}r(s^{'},a^{'})|_{a^{'}=\mu_{\theta}(s^{'})}\\
&+\gamma \nabla_{a^{'}} T(s^{'},a^{'})|_{a^{'}=\mu_{\theta}(s^{'})}\nabla_{s^{''}}V^{\mu_{\theta}}(s^{''})|_{s^{''}=T(s^{'},a^{'})})ds^{'} \\
&+  \sum_{s''}^{}\gamma^{2} I(s,s^{''},2,\mu_{\theta})\nabla_{\theta}V^{\mu_{\theta}}(s^{''})ds^{''}.
\end{split}
\end{equation}

By unrolling (\ref{gra_v}) with infinite steps, we get (\ref{close_v}).
\end{proof}

\section{C. Implementation Details}
In this section we describe the details of the implementation of DVPG, DVG, DDPG and DDPG (model).  The configuration of the actor network and the critic network is the same as the implementation of OpenAI Baselines. For the reward network, we use the same network structure.
Each network has two fully connected layers, where each layer has 64 units. 
The activation function is ReLU, the batch size is $128$, the learning rate of the actor is ${10}^{-4}$, and the learning rate of the critic is ${10}^{-3}$.  The learning rates of the transition network and the reward network are all ${10}^{-3}$. We also add $L_2$ norm regularizer to the loss.

For the reward network, the loss is  $\frac{ \sum_{j=1}^{n}(r(s_j,a_j|\theta_r)-r_j)^2}{n}$. For the transition network, the loss is $\frac{ \sum_j||T(s_j,a_j|\theta_T)-s_{j+1}||_2^2}{n}.$

We also compare with DDPG with model-based rollouts, i.e., besides the training of the policy on real samples, the actor is also updated by model generated samples. The detail of DDPG(model) is referred to Algorithm \ref{alg:DDPG(model)}.

\begin{algorithm*}[!h]
  \caption{The DDPG with model-based rollouts algorithm}
  \begin{algorithmic}[1]
 \STATE Initialize the reward network $r'$, transition network $T'$, critic network $Q$, actor network $\mu_{\theta}$ and target networks $Q'$, $\mu_{\theta}'$\\
 \STATE Initialize experience replay buffer $\mathcal{B}$ and model-based experience replay buffer $\mathcal{MB}$\\
  \FOR{episode$=0, ..., N-1$} 
    \FOR{$t=1, ..., T$} 
  \STATE  Select action according to the current policy and exploration noise\\
  \STATE  Execute action $a_t$, observe reward $r_t$ and new state $s_{t+1}$, and store transition $(s_t,a_t,r_t,s_{t+1})$ in $\mathcal{B}$\\
  \STATE  Sample a random minibatch of $n$ transitions from $\mathcal{B}$\\
  \STATE  Update the critic $Q$ by minimizing the TD error: $$\frac{1}{n}\sum_{i}^{n}{(r_i+\gamma Q'(s_{i+1},\mu_{\theta}(s_{i+1})) - Q(s_i,a_i))}^{2}$$\\
  \STATE  Update the reward network $r'$ and the transition network $T'$ on the batch by minimizing the square loss\\
  \STATE Estimate the policy gradients by \begin{equation}
\sum_{i=1}^{n}\frac{\nabla_{\theta}\mu_{\theta}(s_i)
\nabla_{a_i^{'}}Q^{\mu_{\theta}}(s_i,a_i^{'})|_{a_i^{'}=\mu_{\theta}(s_i)}}{n}.
  \end{equation}
    \STATE Perform model-free gradients update on the policy\\
    \STATE Update the target networks
    \ENDFOR
     \STATE Generate $K$ samples by the policy and the learned transitions, and store fictitious samples in $\mathcal{MB}$\\
      \FOR{$t=1, ..., a$}
      \STATE Sample a random minibatch of $n$ transitions from $\mathcal{MB}$\\
      \STATE Estimate the policy gradients on fictitious samples\\
      \STATE Perform model-based gradients update on the policy\\
      \ENDFOR 
      \STATE Reset the model-based buffer $\mathcal{MB}$ to be empty\\
  \ENDFOR
\end{algorithmic}
  \label{alg:DDPG(model)}
\end{algorithm*}

For the running time of the DVPG algorithm, it takes about 4 hours for running 1M steps.

\section{D. Comparisons with start of the art stochastic policy optimization methods}

\begin{figure*}[!h]
    \centering
    \subfigure[LunarLander-v2.]{
        \begin{minipage}[t]{.3\linewidth}
            \centering
            \includegraphics[scale=0.3]{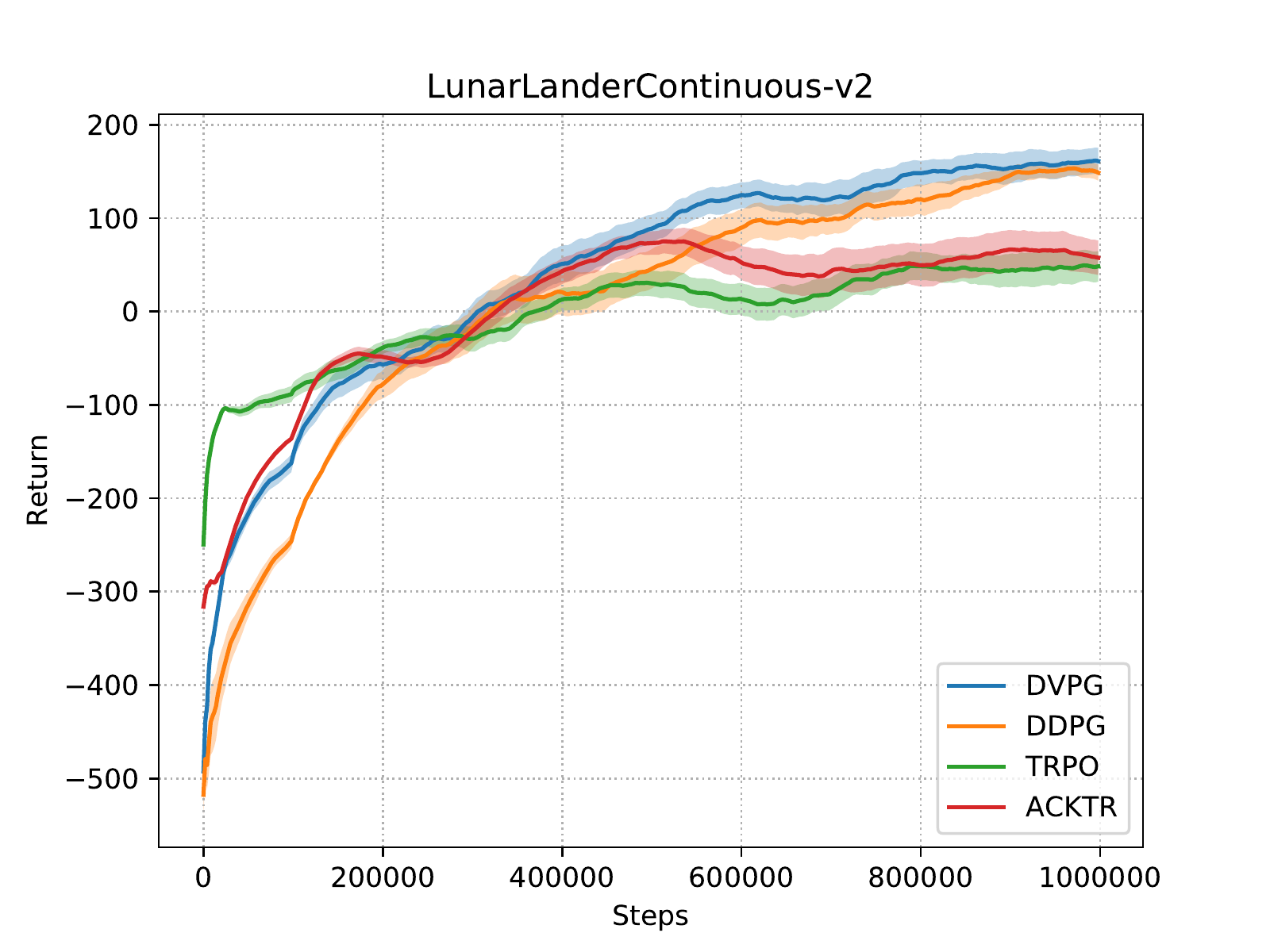}
        \end{minipage}
    }%
    \subfigure[Swimmer-v2.]{
        \begin{minipage}[t]{.3\linewidth}
            \centering
            \includegraphics[scale=0.3]{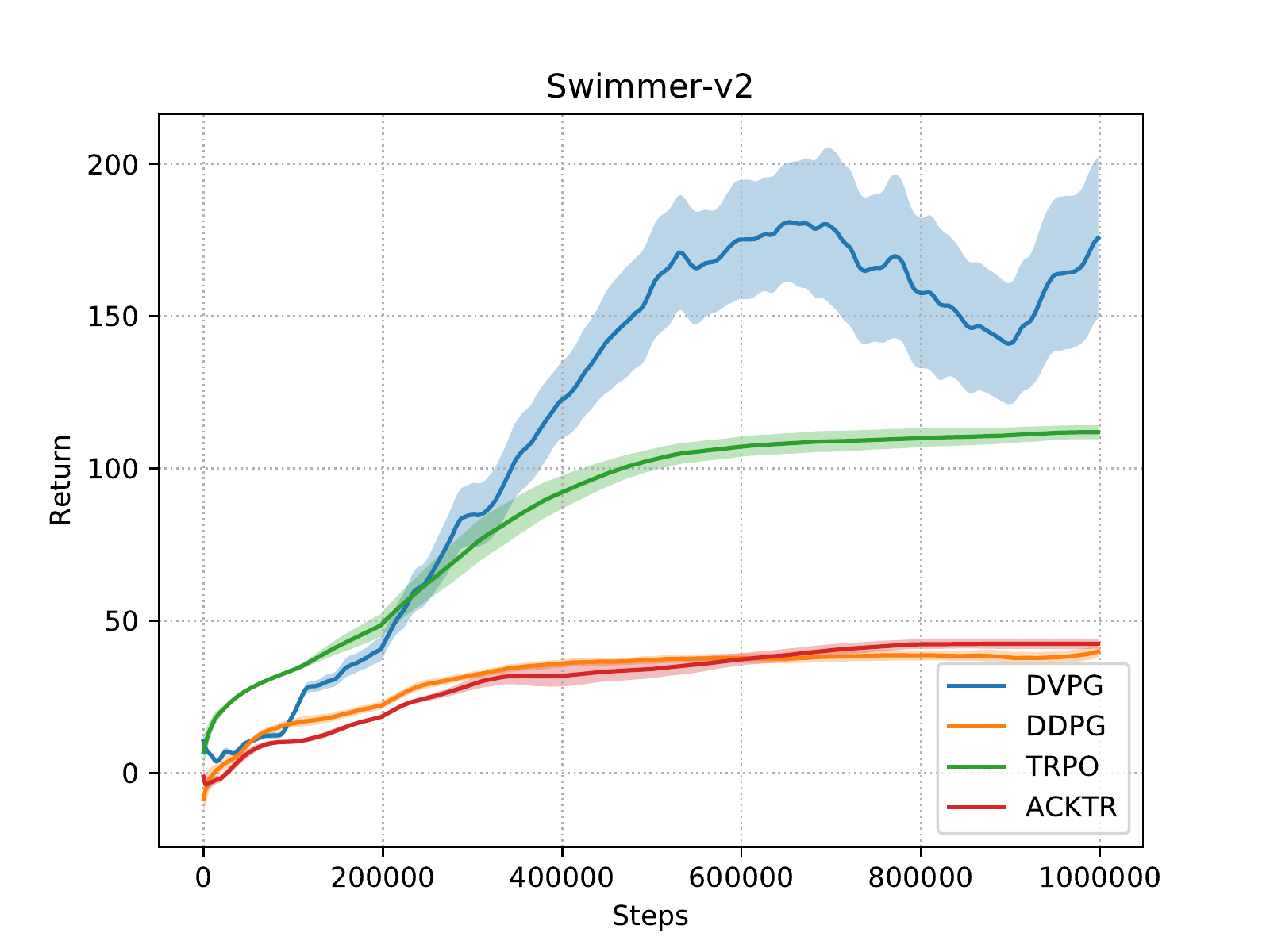}
        \end{minipage}
    }%
    \subfigure[HalfCheetah-v2.]{
        \begin{minipage}[t]{.3\linewidth}
            \centering
            \includegraphics[scale=0.3]{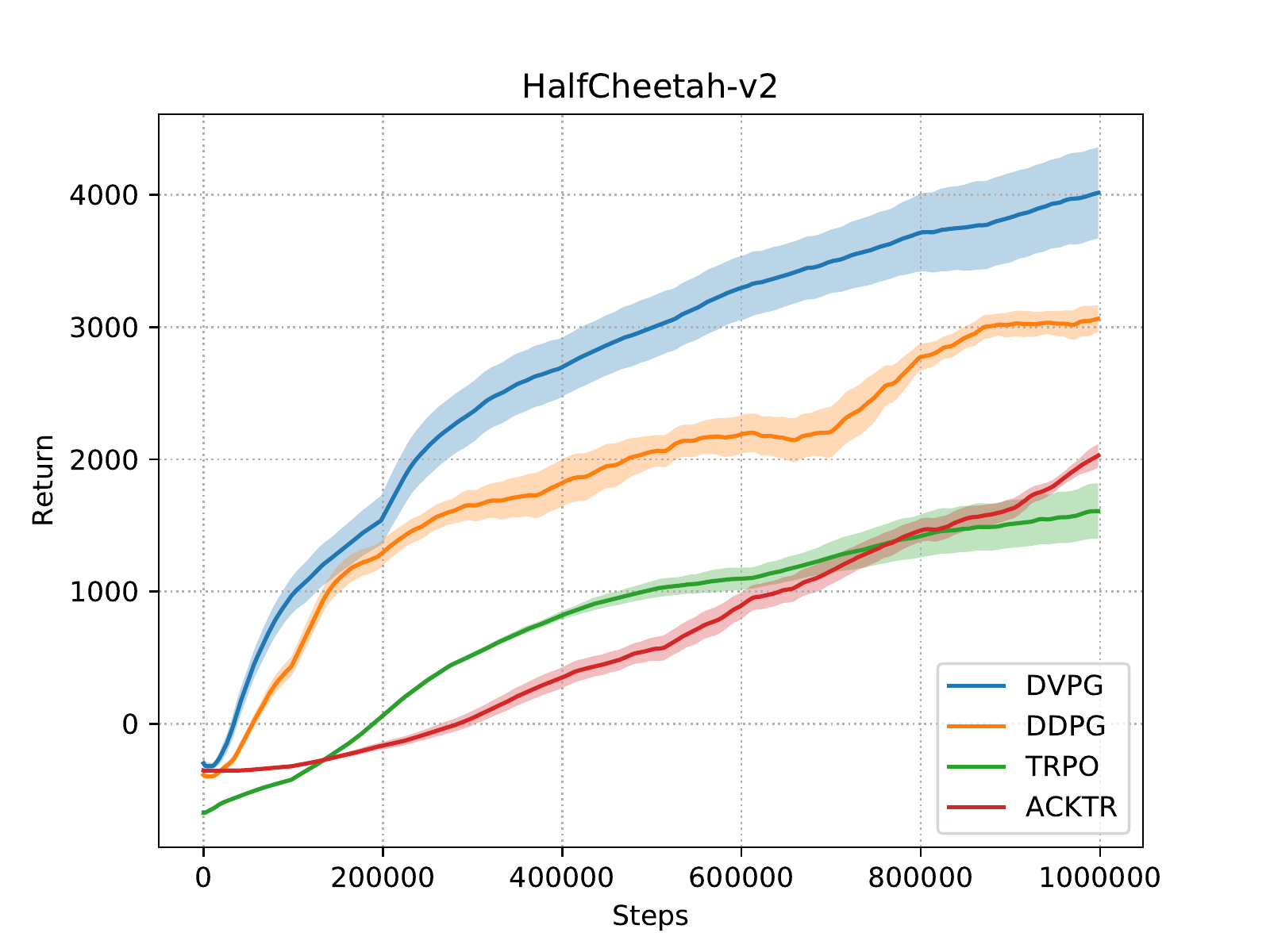}
        \end{minipage}
    }
    \subfigure[HumanoidStandup-v2.]{
        \begin{minipage}[t]{.3\linewidth}
            \centering
            \includegraphics[scale=0.3]{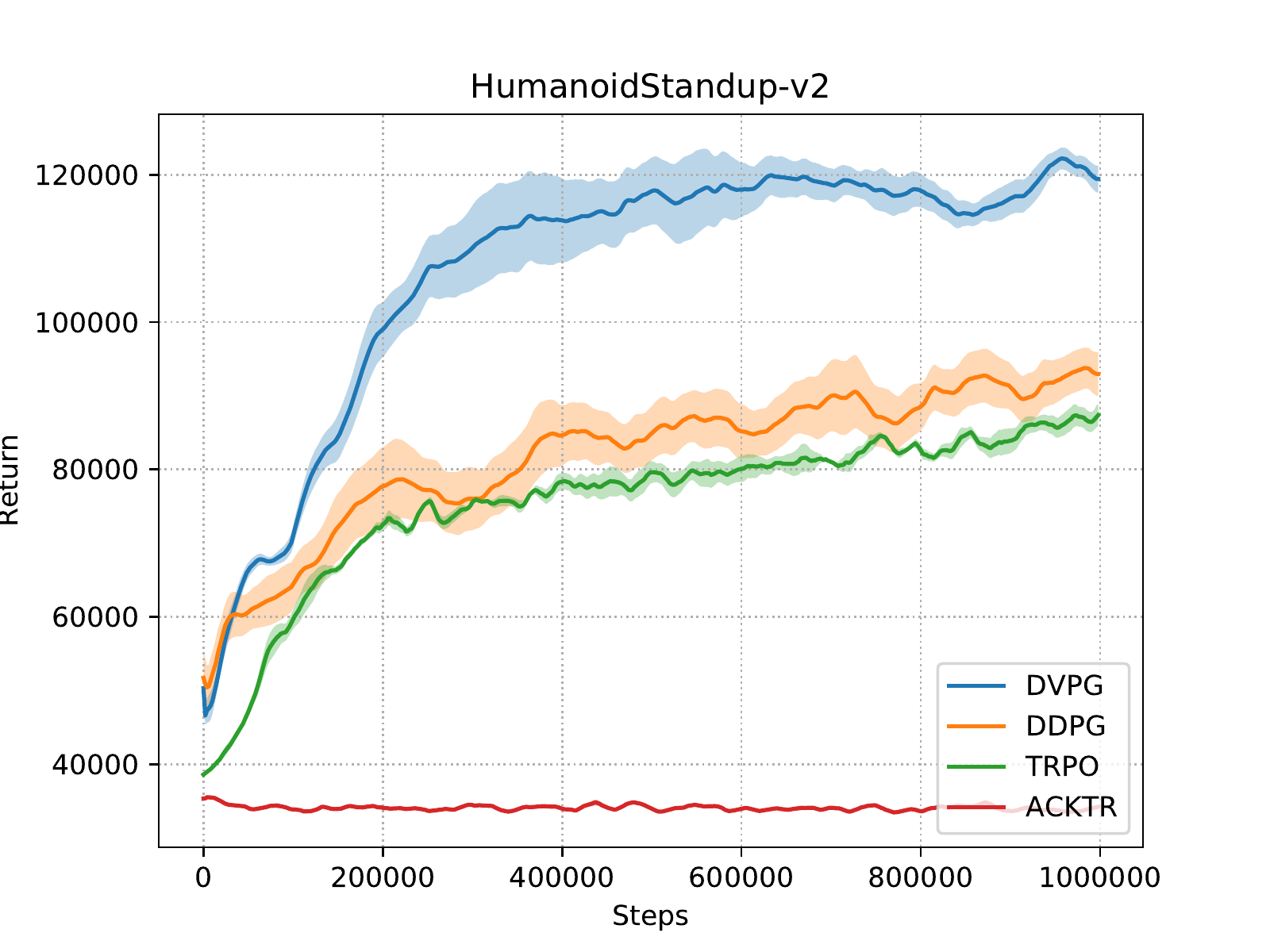}
        \end{minipage}
    }%
    \subfigure[Humanoid-v2.]{
        \begin{minipage}[t]{.3\linewidth}
            \centering
            \includegraphics[scale=0.3]{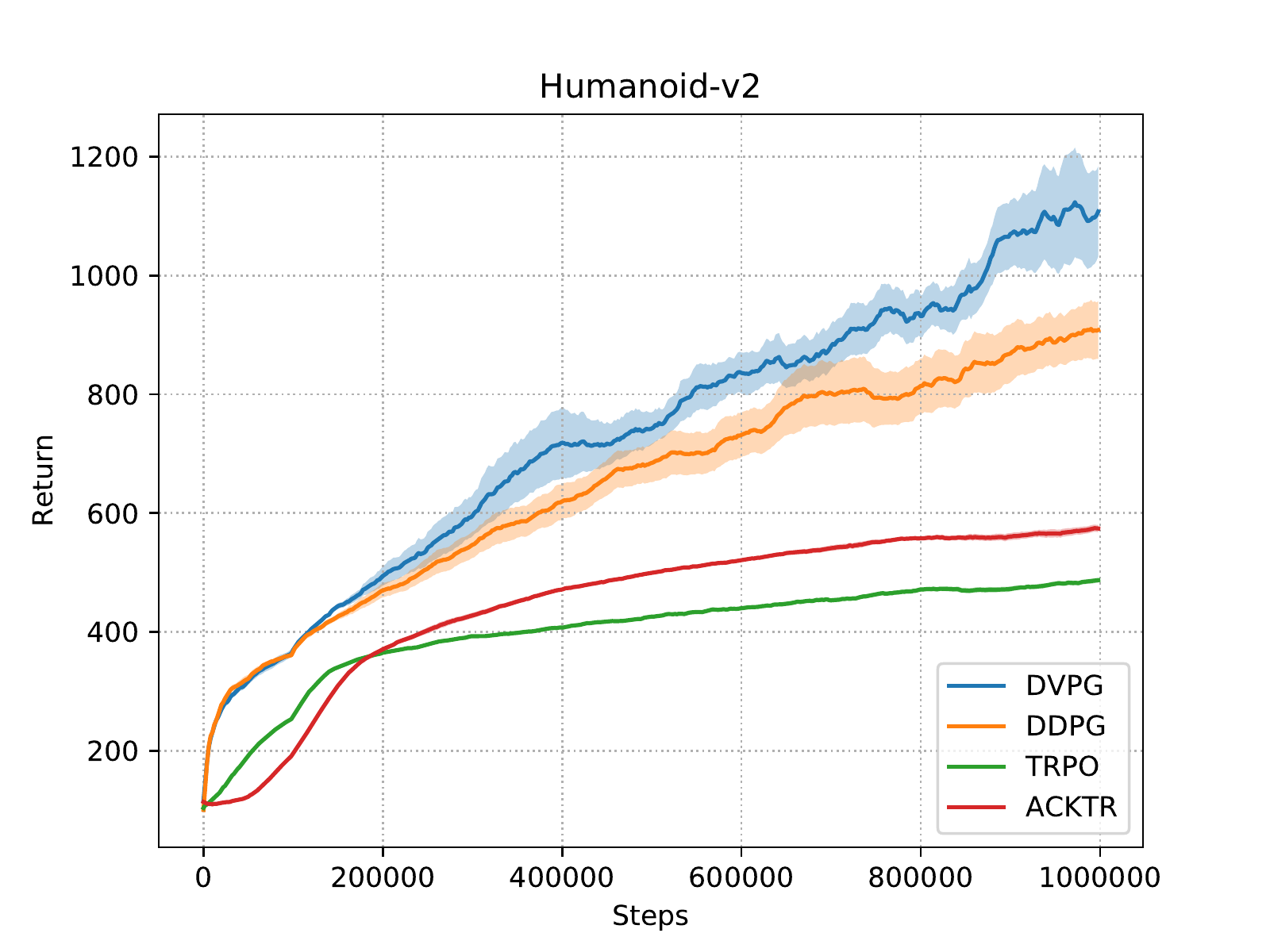}
        \end{minipage}
    }
    \caption{Return/steps of training on environments from the MuJoCo simulator.}
    \label{fig: mujoco}
\end{figure*}

We compare the DVPG algorithm and the DDPG algorithm with state of the art stochastic policy optimization algorithms, TRPO and ACKTR. As shown in Figure \ref{fig: mujoco}, results show that DVPG performs much better than DDPG and other algorithms in the environments where DDPG is more sample efficient than policy optimization algorithms. DVPG also outperforms other baselines significantly in Swimmer-v2 where DDPG is outperformed by TRPO.

\bibliography{dvpg}
\bibliographystyle{aaai}

\end{document}